\documentclass{article}
\usepackage{ijcai23}

\usepackage{times}
\usepackage{soul}
\usepackage{url}
\usepackage[hidelinks]{hyperref}
\usepackage[utf8]{inputenc}
\usepackage[small]{caption}
\usepackage{graphicx}
\usepackage{amsmath}
\usepackage{amsthm}
\usepackage{booktabs}
\usepackage{algorithm}
\usepackage{algorithmic}
\usepackage[switch]{lineno}
 \usepackage{amsmath}
 \usepackage{txfonts}
 \usepackage{graphicx}
 \usepackage{pifont}
 \usepackage{caption}
\usepackage{graphicx}
\usepackage{float} 
\usepackage{subcaption}
\newtheorem{theorem}{Theorem}

\title{Deep  Multi-View  Subspace  Clustering  with  Anchor Graph}
\author{Chenhang Cui, Yazhou Ren, Jingyu Pu, Xiaorong Pu,  Lifang He}

\author{
    {Chenhang Cui$^{1}$, Yazhou Ren$^{1}$\thanks{Corresponding author.}, Jingyu Pu$^1$,  Xiaorong Pu$^1$, Lifang He$^2$}\\
   {\textmd{$^1$ University of Electronic Science
and Technology of China, Chengdu 611731, China}}\\
    {\textmd{$^2$  Lehigh University, PA 18015, USA.}}\\
    {\textmd{osallymalone@gmail.com,  yazhou.ren@uestc.edu.cn, pujingyu0105@163.com,\\  puxiaor@uestc.edu.cn, lih319@lehigh.edu}}
}
\begin{document}
\maketitle
\begin{abstract}

Deep multi-view subspace clustering (DMVSC) has recently attracted increasing attention due to its promising performance. However, existing DMVSC methods still have two issues: (1) they mainly focus on using autoencoders to nonlinearly embed the data, while the embedding may be suboptimal for clustering because the clustering objective is rarely considered in autoencoders, and (2) existing methods typically have a quadratic or even cubic complexity, which makes it challenging to deal with large-scale data. To address these issues, in this paper we propose a novel deep multi-view subspace clustering method with anchor graph (DMCAG). To be specific, DMCAG firstly learns the embedded features for each view independently, which are used to obtain the subspace representations. To significantly reduce the complexity, we construct an anchor graph with small size for each view. Then, spectral clustering is performed on an integrated anchor graph to obtain pseudo-labels. To overcome the negative impact caused by suboptimal embedded features, we use pseudo-labels to refine the embedding process to make it more suitable for the clustering task. Pseudo-labels and embedded features are updated alternately. Furthermore, we design a strategy to keep the consistency of the labels based on contrastive learning to enhance the clustering performance. Empirical studies on real-world datasets show that our method achieves superior clustering performance over other state-of-the-art methods.

\end{abstract}
\section{Introduction}
Subspace clustering has been studied extensively over the years, which assumes that the data points are drawn from low-dimensional subspaces, and could be expressed as a linear combination of other data points. Especially, sparse subspace clustering (SSC)  \cite{elhamifar2013sparse}  has shown the ability to find a sparse representation corresponding to the points from the same subspace. After obtaining the representation of the subspace, the spectral clustering is then applied to obtain the final clustering results. On the other hand, low-rank subspace segmentation was proposed in \cite{liu2012robust} to find a low-rank subspace   representation. 
Despite some state-of-the-art performances have been achieved, most existing methods only focus on single-view clustering tasks.
%The above-mentioned methods have produced promising performances, however, most of them only focus on improving the single-view clustering performance.
 
In many real-world applications, with the exponential growth of data, the description of data has gradually evolved from a single source to multiple sources. For example, a video consists of text, images, and audio. A piece of text can be translated into various languages, and scenes can also be described from different perspectives. These different views often contain complementary information to each other. Making full use of the complementary and consistent information among multiple views could potentially improve the clustering performance.

Considering the diversity of information that comes with multi-view data, the research of multi-view subspace clustering (MVSC) has attracted increasing attention recently.  MVSC aims to seek a unified subspace from learning the fusion representation of multi-view data, and then separates data in the corresponding subspace. In the literature, many MVSC methods have been proposed \cite{zhang2015low,luo2018consistent,li2019flexible,wang2019multi,zheng2020feature,liu2021multi,si2022consistent}. However, one major weakness of existing approaches is their high time and space complexities, which are often quadratic or cubic in the number of samples $n$.
Recently, a number of anchor-based multi-view subspace clustering methods \cite{chen2011large,sun2021scalable,kang2020large,wang2022align,liu2022efficient} have been developed, which can achieve promising performance with a large reduction in storage and computational time. Generally, the anchor graphs are equally weighted and fused into the consensus graph, and then spectral clustering is performed to obtain the clustering result. 

On the other hand, inspired by deep neural networks (DNNs), many deep multi-view subspace clustering (DMVSC) methods have been proposed 
\cite{peng2020deep,wang2020deep,kheirandishfard2020multi,sun2019self,zhu2019multi,ji2017deep}. However, most DMVSC methods only consider the feature learning ability in networks, their performance is still limited because this learning process is typically independent of the clustering task.
%They are still limited without the correct unsupervised learning process.
 
To address the above issues, this paper proposes deep multi-view subspace clustering with anchor graph (DMCAG). DMCAG firstly utilizes deep autoencoders to learn low-dimensional embedded features by optimizing the reconstruction loss for each view independently. For each view, a set of points are chosen by performing $k$-means on the learned features to construct anchor graphs. Then, we utilize anchor graphs and embedded features as input to get the subspace representation respectively. Once the desired subspace representation is obtained, the clustering result  can be calculated by applying the standard spectral clustering algorithm. Unlike the most existing DMVSC methods, the proposed method does not output the clustering result from spectral clustering directly. 
Instead, we obtain a unified target distribution from this clustering result primarily, which is more robust than that generated by $k$-means \cite{xie2016unsupervised,xu2022self}, especially for clusters that do not form convex regions or that are not clearly separated. In a self-supervised manner, the Kullback-Leibler (KL) divergence between the unified target distribution and each view's cluster assignments is optimized. We iteratively refine embedding with pseudo-labels derived from the spectral clustering, which in turn help to obtain complementary information and a more accurate target distribution. Besides, to ensure the consistency among different views and avoid affecting the quality of reconstruction, we adopt contrastive learning on the labels instead of latent features. 
The main contributions of this paper are summarized as follows:
\begin{itemize}
    \item We propose a novel deep self-supervised model for MVSC. A unified target distribution is generated via spectral clustering which is more robust and can accurately guide the feature learning process. The target distribution and learned features are updated iteratively. 
    \item To boost the model efficiency, we use anchor graph to construct the graph matrix, avoiding constructing a $n$ × $n$ graph. This strategy can significantly reduce time complexity by sampling anchor points. %, which relatively practical especially dealing with large-scale problems.
    \item %We design a scheme based on contrastive learning to obtain consistent soft cluster assignments in multiple views.
    We utilize contrastive learning on pseudo-labels to alleviate the conflict between the consistency objective and the reconstruction objective, thus consistent soft cluster assignments can be obtained among multiple views.
    \item Extensive experiments on real-world data sets validate the effectiveness and efficiency of the proposed model.
\end{itemize}

\section{Related Work}
\subsection{Deep Embedded Multi-View Clustering}
In recent years, the application of deep learning technology
in multi-view clustering has been a hot topic. Deep embedded clustering (DEC) \cite{xie2016unsupervised} utilizes the autoencoder to extract the low-dimensional latent representation from raw features  and then optimizes the student’s $t$-distribution and target distribution of the feature representation to achieve clustering. 
In contrast, traditional multi-view clustering algorithms mostly use linear and shallow embedding to learn the latent structure of multi-view data. 

However, these methods cannot utilize the nonlinear property of data availably, which is crucial to reveal a complex clustering structure \cite{ren2022deep}. Deep embedded multi-view clustering with collaborative training (DEMVC) \cite{xu2021deep} is a novel framework for multi-view clustering, in which a shared scheme of the auxiliary distribution is used to improve the performance of the clustering. 
By assuming that clustering structures with high discriminability play a significant role in clustering, self-supervised discriminative feature learning for multi-view clustering (SDMVC) \cite{xu2022self} leverages the global discriminative information contained in all views' embedded features. During the process, the global information will guide the feature learning process of each view. The autoencoder is usually utilized to capture the most important features present in the data. A suitable autoencoder can obtain more robust representations. Deep embedding clustering based on contractive autoencoder (DECCA) \cite{diallo2021deep} simultaneously disentangles the problem of learned representation by preserving important information from the initial data while pushing the original samples and their augmentations together. With the introduction of the contractive autoencoders, the learned features are better than normal autoencoders. The available information provided by latent graph is often ignored, deep embedded multi-view clustering via jointly learning latent representations and graphs (DMVCJ) \cite{huang2022deep} utilizes the graphs from latent features to promote the performance of deep MVC models. By learning the latent graphs and feature representations jointly, the available information implied in latent graph can improve the clustering performance. The self-supervised manner is widely used in deep  clustering, which is valuable to migrate it to other clustering algorithms.

\subsection{Multi-View Subspace Clustering}
Although many methods exist in subspace clustering, such as low-rank representation subspace clustering (LRR) \cite{liu2012robust}, sparse subspace clustering (SSC) \cite{elhamifar2013sparse}, most of the multi-view subspace clustering methods adopt self-representation to obtain the subspace representation. Low-rank tensor constrained multi-view subspace clustering (LMSC) \cite{zhang2017latent} learns the latent representation based on multi-view features, and generates a common subspace representation rather than that of individual view. Flexible multi-view representation learning for subspace clustering (FMR) \cite{li2019flexible} avoids using partial information for data reconstruction and makes the latent representation well adapted to subspace clustering. 

Most of existing MVSC methods are challenging to apply in large-scale data sets. Fortunately, inspired by the idea of anchor graph which can help reduce both storage and computational time, a lot of anchor-based MVSC methods have been proposed. Large-scale multi-view subspace clustering (LMVSC) \cite{kang2020large} can be solved in linear time, which selects a small number of instances to construct anchor graphs for each view, and then integrates all anchor graphs from each views. It thus can perform spectral clustering on a small graph. Efficient one-pass multi-view subspace clustering with consensus anchor (CGMSC) \cite{liu2022efficient} unifies fused graph construction and anchor learning into a unified and flexible framework so that they seamlessly contribute mutually and boost performance. Fast multi-view anchor correspondence clustering (FMVACC) \cite{wang2022align} finds that the selected anchor sets in multi-view data are not aligned, which may lead to inaccurate graph fusion and degrade the clustering performance, so an anchor alignment module is proposed to solve the anchor-unaligned problem (AUP). 

With the development of deep neural networks, a number of deep multi-view subspace clustering methods have been proposed. Deep subspace clustering with $\ell_1$-norm (DSC-L1) \cite{peng2020deep} learns nonlinear mapping functions for data to map the original features into another space, and then the affinity matrix is calculated in the new space. Deep multi-view subspace clustering with unified and discriminative learning (DMSC-UDL) \cite{wang2020deep} integrates local and global structure learning simultaneously, which can make full use of all the information of the original multi-view data. Different from the above-mentioned deep multi-view clustering methods, we combine representation learning and spectral clustering into a unified optimization framework, the clustering results can be sufficiently exploited to guide the representation learning for each view.

\section{Method}
\paragraph{Problem Statement. }Given multi-view data $X= \{X^v \in \mathbb{R}^{d_v \times n }\}_{v=1}^V$ with $V$ views, $d_v$ is the dimension of  the $v$-th view, and $n$ is the instance number. The target of MVSC is to divide the given instances into $k$ clusters. 
\subsection{Motivation}
Subspace clustering aims to find out an underlying subspace which expresses each  point as a linear combination of other points. The final clustering assignment is obtained by performing spectral clustering on the learned subspace. 
Basically, it can be mathematically denoted as:
\begin{equation}
\begin{aligned}
\mathop{\min}_{X^v} &{\sum_{v=1}^{V}||X^v-X^vS^v ||_F^2}+\gamma ||S^v||^2_F  \\
   &s.t.\ \  S^v \geqslant 0 , S^v \textbf{1}=\textbf{1}, 
\end{aligned}
\end{equation}
where $S^v\in R^{n \times n}$ is the learned subspace of $v$-th view and $\gamma $ is a trade-off coefficient that controls the  sparsity of $S^v$. The constraints on $S^v$ ensure that $S^v$ is non-negative and $\sum_j S_{ij}=1$. When the adjacency graphs are obtained, spectral clustering can be performed on $S^v$ to get the clustering result. 
%However, constructing a $n \times  n$ graph for multi-view data  requires a lot of time and storage. 
Existing MVSC methods aim at learning concrete subspace effectively and mining global information of all views to improve the clustering assignment quality, but still meet some challenges:  

(1) Most MVSC methods require high time complexity (at least $O(n^2 k)$) to calculate the clustering result on raw features, which requires more storage and time and is also difficult to deal with large-scale datasets.
Additionally, some DMVSC methods such as \cite{ji2017deep,li2021self}  focus on  learning  subspace by nonlinearly embedding the data. 
Since the connection between the learned subspace and clustering is weak, the learned features may not be optimal for the clustering task.
%so features' information contained is unclear without concrete supervision in the training process.

(2) From existing self-supervised MVC methods, we can observe that most of them heavily depend on the quality of latent representations to supervise the learning process. \cite{xu2022self} fuses the embedded features of all views and utilizes $k$-means \cite{kmeans1967} to obtain the global target distribution. However, $k$-means does not perform well on some data structure such as non-convex structure compared with spectral clustering  \cite{ng2001spectral} which is more robust to different distributions. Hence, their obtained pseudo-labels may not reflect clear  clustering structure to guide the latent feature learning process. 

(3) Some views of an instance might have wrong clustering assignment, leading to the clustering inconsistency. Some deep-learning based MVC methods achieve the consistency by directly learning common information on latent features \cite{cheng2021multi}, which may reduce the complementarity of each view's information due to the conflict between learning exact latent representations and achieving consistency. 
\par
To address the above-mentioned issues, we propose a novel DMCAG framework as shown in Fig. \ref{framework}. We use the anchor method to construct the graph matrix on latent features of each view which requires less time and storage. After that, we obtain global pseudo-labels by performing spectral clustering on the integrated anchor graph. Since spectral clustering is more robust to the data distribution, our self-supervised process can learn higher discriminative information from each view.
%from graphs constructed from spectral clustering, which is more robust to the different data distribution, and then the self-supervised process learns higher discriminative information of each view. 
Then, we adopt contrastive learning  on pseudo-labels of latent features to maintain the view-private information and achieve the clustering consistency among all views.  

\begin{figure*}[!t] 
    \centering % 
    \includegraphics[width=18cm]{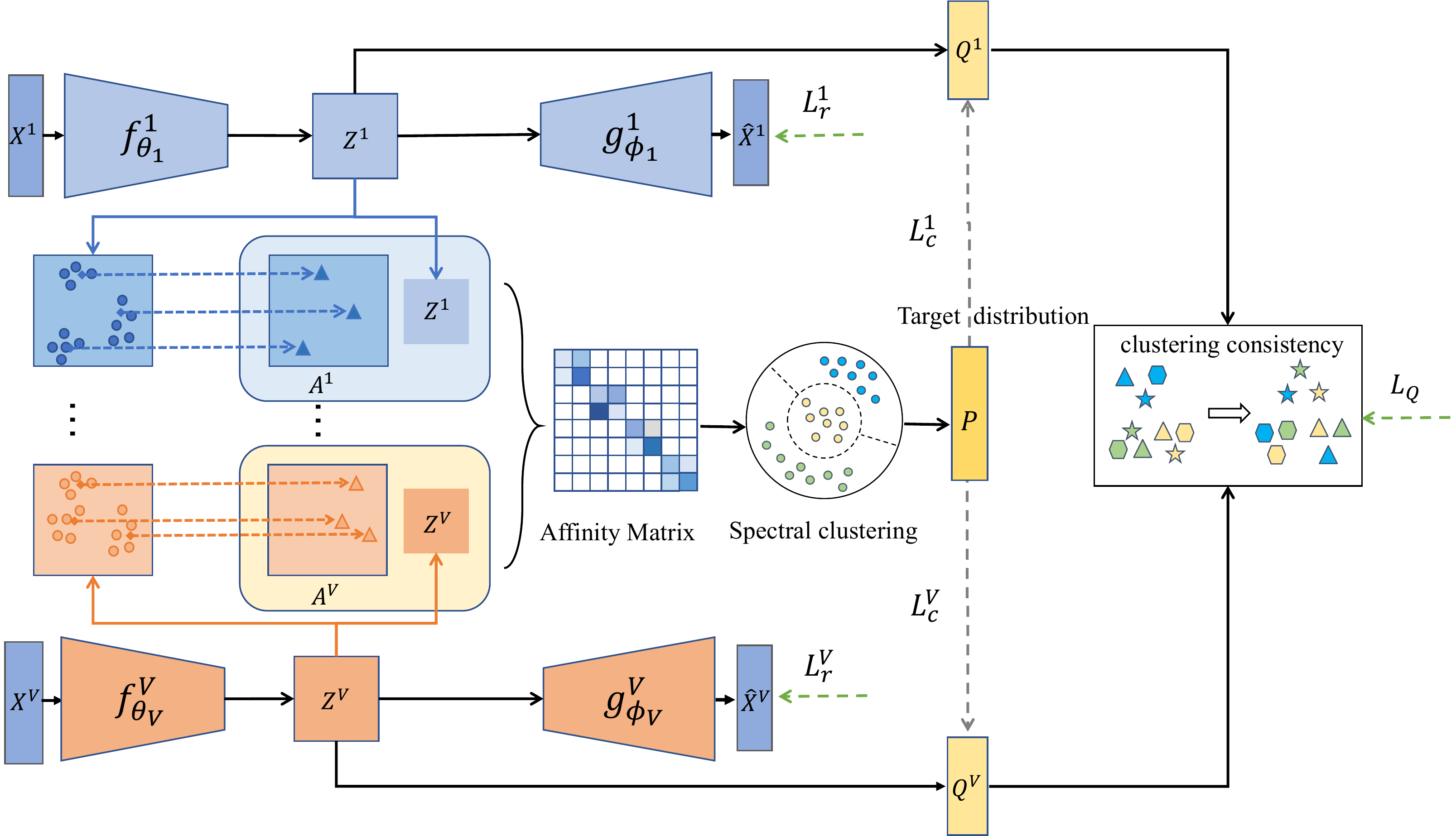}
    \caption{The framework of DMCAG. For the $v$-th view, $X^v$ denotes the input data, $Z^v$ denotes the embedded features, $A^v$ is a set of clustering centroids with the number of $m$,   and $Q^v$ is the soft cluster assignment distribution. $P$ denotes the unified target distribution obtained through spectral clustering.}
    \label{framework}
\end{figure*}

\subsection{ Learning Anchor Graph via Autoencoders }
As for much redundancy in the raw data, we utilize the deep autoencoder to extract the latent representations of all views. Through the encoder $\ f_{\theta^v}^v$ and decoder $\ g_{\phi^v}^v$, where $\theta^v$ and $\phi^v$ are learnable parameters, $X^v$ is encoded as $Z^v \in R^{l\times n}$ ($l$ is the same for all views) via $\ f_{\theta^v}^v$ and $Z^v$ is decoded as $\hat{X}^v$ via $g_{\phi^v}^v$. The reconstruction loss is defined as:
\begin{equation}
\label{l_r}
{L}_r=\sum_{v=1}^V{L}_r^v=\sum_{v=1}^V ||X^v - g_{\phi^v}^v(f_{\theta^v}^v(X^v)) ||_F^2. 
\end{equation}Inspired by \cite{kang2020large}, we adopt the anchor graph to replace the full adjacency matrix $S$, which is formulated as:
\begin{equation}
\label{l_a}
\begin{aligned}
\mathop{\min}_{C^v} {L}^v_a&={\sum_{v=1}^{V}||Z^v-A^v (C^v)^T ||_F^2}+\gamma ||C^v||^2_F  \\
&s.t. \ \  C^v \geqslant 0 , (C^v)^T \textbf{1}=\textbf{1} ,
\end{aligned}
\end{equation}
where $A^v \in R^{l\times m}$ ($m$ is the anchor number) is a set of  clustering centroids through $k$-means on the embedded features $Z^v$, and $C^v\in R^{n\times m}$ is the anchor graph matrix reflecting relationship between $Z^v$ and $A^v$. The problem above can be solved by convex quadratic programming. We refer the readers to \cite{wolfe1959simplex} for more details about quadratic programming. 
 \subsection{Spectral Self-Supervised Learning }
As \cite{ng2001spectral} shows, for clusters that are not clearly  separated or do not form convex regions, the spectral method  can also reliably find clustering assignment. Hence, we use spectral clustering to obtain more robust global target distribution to guide the self-training. 
Spectral clustering \cite{ng2001spectral} can be mathematically described  as finding   $Q\in R^{n\times k}$ by maximizing:
\begin{equation}
\label{spec}
\mathop{\max}_{Q}  Tr\ (Q^TSQ) \ \ s.t.\  Q^TQ=I.
\end{equation}
Following Theorem \ref{th1} introduced by \cite{chen2011large,kang2020large},  we present an approach to approximate the singular vectors of  $S$  in latent space.
\begin{theorem} 
\label{th1} \cite{chen2011large,kang2020large}
Given a similarity matrix $S$, which can be decomposed as $(C^T)C$. Define singular value decomposition (SVD) of $C$ as $U{\wedge}V^T$, then we have 
\begin{equation}
\label{t1}
\mathop{\max}_{Q^TQ=I}  Tr\ (Q^TSQ) \Longleftrightarrow   \mathop{\min}_{Q^TQ=I,H}  ||C-QH^T||^2_F. 
\end{equation}
And the optimal solution $Q^{*}$ is equal to U.
\end{theorem}
\begin{proof}
From Eq. (\ref{t1}),  one can observe that the optimal  $H^{*}=C^TQ$. Substituting  $H^{*}=C^TQ$ into Eq. (\ref{t1}), the following equivalences hold 
\begin{align*}
 & \mathop{\min}_{Q^TQ=I,H}  ||C-QH^T||^2_F   \Longleftrightarrow  \mathop{\min}_{Q^TQ=I}  ||C-QQ^TC||^2_F \\
 & \Longleftrightarrow  \mathop{\max}_{Q^TQ=I}  Tr\ (Q^TCC^TQ) 
 \Longleftrightarrow \mathop{\max}_{Q^TQ=I}  Tr\ (Q^TSQ). 
% \mathop{\min}_{Q^TQ=I,H}  ||C-QH^T||^2_F  & \Longleftrightarrow % \mathop{\min}_{Q^TQ=I}  ||C-QQ^TC||^2_F \\
% & \Longleftrightarrow  \mathop{\max}_{Q^TQ=I}  Tr\ (Q^TCC^TQ) \\
% &\Longleftrightarrow \mathop{\max}_{Q^TQ=I}  Tr\ (Q^TSQ). 
\end{align*}
Furthermore, one can obtain 
\begin{align*}
   S  = CC^T &=U{\wedge}V^T(U{\wedge}V^T)^T 
      =U{\wedge}(V^TV){\wedge}U^T 
      =U{\wedge}^2 U^T.
%   S  = CC^T &=U{\wedge}V^T(U{\wedge}V^T)^T \\
%    & =U{\wedge}(V^TV){\wedge}U^T \\
%    &=U{\wedge}^2 U^T.
\end{align*}
Therefore, we could use  left singular vectors of $C$   to approximate the eigenvectors of $S$.
\end{proof}

According to Theorem \ref{th1}, we calculate eigenvectors $\{U^v\}_{v=1}^V$ of $\{C^v\}_{v=1}^V$ to approximate the eigenvectors of full similarity matrix. To fully exploit the complementary information across all views,  we concatenate all eigenvectors ${U}=\{U^1, U^2, \dots, U^V\}\in R^{n\times(V k)}$  to generate the global feature via the spectral method. After obtaining the global feature ${U}$,  we apply $k$-means to calculate the  cluster centroids $\{  {\mu_j}\}_{j=1}^k$: 
\begin{equation}
\label{globalk}
  \min _{{\mu_1}, \ldots,{\mu_k}} \sum_{i=1}^{n} \sum_{j=1}^{k}\left\|{U}{(i,:)}- {\mu_j}\right\|^{2}.
\end{equation}
Similiar to DEC \cite{xie2016unsupervised}, which is a popular single-view deep clustering method utilizing Student’s t-distribution \cite{van2008visualizing}, the soft clustering assignment $t_{ij}$  between global feature ${U} $ and  each cluster  centroid $ {\mu_j}$ could be computed as:
\begin{equation}
    t_{i j}=\frac{(\alpha+\left\|{U}{(i,:)}-{\mu_j}\right\|^{2})^{-1}}{\sum_{j}(\alpha+\left\|{U}{(i,:)}-{\mu_j}\right\|^{2})^{-1}}.
\end{equation}

To increase the discriminability of the global soft assignments, the target distribution $P$ is formulated as:
\begin{equation}
\label{o}
    p_{i j}=\frac{(t_{i j}^{2} / \sum_{i} t_{i j})}{\sum_{j}(t_{i j}^{2} / \sum_{i} t_{i j})}.
\end{equation}
%According to the value of ${U}$, we set $\alpha$ to 0.001.

We obtain soft clustering assignment (pseudo-label) $Q^{v}=\left[q_{1}^{v}, q_{2}^{v}, \ldots , q_{N}^{v}\right]$  of each view, where $q_{i j}^{v}$  can be  considered as  the probability of the $i$-th instance  belonging  to the $j$-th cluster in the $v$-th view.
It is defined as:
\begin{equation}\label{C}
    q_{i j}^{v}=\frac{ (1+\left\|{z}_{i}^{v}-\mu_{j}^{v}\right\|^{2} )^{-1}}{\sum_{j} (1+\left\|{z}_{i}^{v}-\mu_{j}^{v}\right\|^{2} )^{-1}},
\end{equation}
where $\mu_j^{v}$ denotes the $j$-th cluster centroid of the $v$-th view.

Overall, we use Kullback-Leibler divergence between the unified target distribution $P$ and view-private soft assignment distribution $Q^v$ to guide autoencoders to optimize latent features containing higher discriminative information, which can be formulated as:
\begin{equation}
\label{l_c}
    L^v_{c}= D_{K L}\left(P \| Q^{v}\right)=\sum_{v=1}^{V}\sum_{i=1}^n\sum_{j=1}^{k} p_{i j} \log \frac{p_{i j}}{q_{i j}^{v}}.  
\end{equation}

As the target distribution obtained from spectral clustering is  adaptive to different data distributions, we can get more explicit cluster structures to guide the self-training process  compared with $k$-means. To extract embedded features that reflect correct information of raw features and learn an accurate assignment for clustering, we jointly optimize the reconstruction of autoencoders  and self-supervised learning. The total loss function ${L_s}$ is defined as:
\begin{equation}
    {L_s}= \sum_{v=1}^V({L_{r}^{v}}+{L_{c}^{v}} ). 
    \label{ls}
\end{equation}
\subsection{Label Consistency Learning}
To guarantee the same soft assignment distribution of all views represent the same cluster, we need to achieve the consistency of pseudo-labels. We adopt contrastive learning to  the soft assignment obtained from 
Eq. (\ref{C}). For the $m$-th view, $Q^m( : ,j)$ have 
 $(Vk-1)$ pairs, where the $(V-1)$  pairs $\{Q^m( : ,j),Q^n( : ,j)\}_{m\neq n}$ are positive and the rest $V(k-1)$  pairs are negative. Thereby the contrastive loss can be defined as:
\begin{equation}
    {L}^{mn}_Q=-\frac{1}{k}   \sum^k_{j=1} \log \frac{e^{d(Q^m( : ,j),Q^n( : ,j))/ \tau}}{\sum^k_{k'=1}\sum_{v=m,n}e^{d(Q^m( : ,j),Q^v( : ,k'))/\tau}- e^{1/\tau}},
\end{equation}
where $d( \cdot,\cdot )$  represents the cosine distance to meausure the similiarty between two labels, $\tau$  is  the temperature parameter. Moreover, to avoid the samples being assigned into a single cluster, we use the cross entropy as a regularization term.
Generally, the label consistency learning is formulated as:
\begin{equation}
\label{l_q}
{L}_Q=\frac{1}{2}\sum^V_{m=1}\sum_{n\neq m} {L}^{mn}_Q +\sum^V_{m=1}\sum^k_{j=1}s_{j}^m \log s_{j}^m, 
\end{equation}
where $s_{j}^m=\frac{1}{N}\sum^N_{i=1} q^m_{ij} $. 
After finetuning the labels via contrastive learning, the similarities of 
positive pairs are enhanced and thereby the obtained latent features  have clearer clustering structure. At last, the  clustering prediction  ${y}$ is obtained through   the target distribution $P$  calculated by performing $k$-means on ${U}$.
\begin{algorithm}[!t]
	\caption{Deep Multi-View Subspace Clustering with Anchor Graph (DMCAG)} 
	\label{alg} 
	\begin{algorithmic}
        \STATE  \textbf{Input}: multi-view dataset $X$ , cluster number $k$.
		\STATE \textbf{Initialization}: Get $\{\theta^v , \phi^v , \mu^v ,A^v  \}_{v=1}^V $ by pretraining autoencoders and $k$-means. Initialize  $\{C^v \}_{v=1}^V$  via  quadratic programming. 
		        \STATE  Update $\{\theta^v , \phi^v\, Q^v \}_{v=1}^V $ by performing self-supervised learning 	via Eq. (\ref{ls}).
        		%\STATE Update $\{A^v\}_{v=1}^V $ according to Eq.(3)
           \STATE Performing   contrastive learning  on $\{Q^v \}_{v=1}^V$ 
via Eq. (\ref{l_q}). 
           \STATE Obtain  $\{C^v \}_{v=1}^V$  via Eq. (\ref{l_a}).
		\STATE  \textbf{Output}: Cluster assignment ${y}$  via Eq. (\ref{globalk}). 
	\end{algorithmic} 
\end{algorithm}
\subsection{Optimization}
The detailed optimization procedure is summarized in Algorithm \ref{alg}. 
We adopt the Adam method to train the autoencoders. At the beginning,  autoencoders are initialized by Eq. (\ref{l_r}). After that, we solve  Eq. (\ref{l_a}) via convex quadratic programming to obtain ${U}$ and calculate the global target distribution $P$. Then, the spectral self-supervised learning is adopted to learn more representive embeddings.  After performing the self-supervised learning, the  contrastive learning is conducted to achieve the clustering consistency. At last, we run $k$-means on   ${U}$ to  obtain the final clustering result ${y}$:

\begin{equation}
{y_i}= \mathop{\mathrm{argmax}}_{j}\ (p_{ij}). 
\end{equation}

\section{Experiments}

\subsection{Experimental Settings}
\begin{table}[!h]
\setlength{\tabcolsep}{0.5mm}
%  \caption{The statistics of the datasets.}
%  \label{datasets}
% \centering 
\centering 
{
\begin{tabular}{|c|c|c|c|c| }
\hline
Dataset & Sample & View & Dimension \\ \hline
MNIST-USPS & 5000 & 2 & [[28,28], [28,28]]\\\hline
Multi-COIL-10 & 720 & 3 & [[32,32], [32,32], [32,32]] \\\hline
BDGP & 2500 & 2 & [1750, 79]  \\ \hline
UCI-digits & 2000 & 3 & [240, 76, 64 ] \\\hline
Fashion-MV & 10000 & 3 & [784, 784, 784]  \\ \hline
HW & 2000 & 6 & [216, 76, 64, 6, 240, 47] \\
\hline
\end{tabular}
}
 \caption{The statistics of experimental datasets.}
 \label{datasets}
\centering 
\end{table}

\begin{table*}[!t]
%  \caption{ Results of all methods on three origin datasets. The best result in each column is highlighted in red and   the second-best is underlined.}
%  \label{raw}
\centering
\begin{tabular}{lcccccccccc}
\hline
  Datasets & $K$-means  &SC   & DEC & CSMSC  &FMR & SAMVC & LMVSC& CGMSC& FMVACC &Ours      \\
  \hline
 \multicolumn{11}{c}{ACC}\\
   \hline
MNIST-USPS&76.78  & 65.96 &73.10   & 72.68&63.02 &69.65&38.54  & 91.22& \underline{98.67}& \textbf{99.58 }\\
 {Multi-COIL-10}& 73.36 &33.75  &74.01  &97.64 & 78.06&84.31 &63.79 & \underline{99.86} & 93.20 & \textbf{100.00} \\
 
BDGP &  43.24&  51.72&  \underline{94.78} & 53.48&95.12 &51.31 &35.85  &45.60&  58.63& \textbf{98.00}\\
 
UCI-digits& 79.50 & 63.35 & 87.35 &88.20 & 80.10 & 74.20 & 74.60& 76.95 & \underline{89.47} &\textbf{95.60}\\

Fashion-MV & 70.93 & 53.54 & 67.07 & 77.61& -& 62.86 & 43.43 &\underline{82.79}&79.62&\textbf{97.89} \\

HW& 75.45 & 77.69 & 81.13 & 89.80&86.05 &76.37 &\underline{91.65} &69.10&89.45 & \textbf{97.90}   \\
\hline
 \multicolumn{11}{c}{NMI}\\
   \hline
MNIST-USPS&72.33  & 58.11 &71.46   & 72.64& 60.90& 60.99 &63.09 &88.24  & \underline{96.74} & \textbf{ 99.86}\\
 Multi-COIL-10& 76.91 &12.31  & 77.43 & 96.17&80.04 &92.09 &75.83 & \underline{99.68} & 93.39& \textbf{100.00}   \\
 
BDGP &56.94  &58.91  &86.92 & 39.92&\underline{87.69} &45.15 & 36.69&27.15& 36.81 & \textbf{94.69}   \\
 
UCI-digits & 77.30  & 66.60 & 79.50 &80.68 &72.13 &74.73 &74.93 & 77.62 &  \underline{84.35} & \textbf{91.10} \\

Fashion-MV &65.61  & 57.72 &72.34 &77.91&- &68.78 & 46.02 & \underline{86.48} & 76.31& \textbf{95.40}  \\

HW&  78.58& 86.91 & 82.61 & 82.95& 76.49& 84.41& 84.43&81.83 & \underline{85.98}& \textbf{95.25}  \\
\hline
 \multicolumn{11}{c}{ARI}\\
   \hline
MNIST-USPS&63.53 & 48.64& 63.23  &64.08 & 49.73& 74.58&27.74   &   86.46&\underline{97.65} &  \textbf{99.07}   \\
 {Multi-COIL-10}&  64.85&  14.02 & 65.66  &94.89 & 70.59 &88.75 &55.05 &\underline{99.69}& 92.47 & \textbf{100.00} \\
 
BDGP &  26.04& 31.56&87.02 & 33.59&\underline{88.38} &19.60 &35.06  &21.99&  44.25& \textbf{95.11}\\
 
UCI-digits& 71.02 & 54.07 & 75.69 &76.72 &65.53 &74.25 &74.27 &70.08  & \underline{83.33}& \textbf{90.59}\\

Fashion-MV &56.89  & 42.61 & 62.91&70.28 &- &56.65  &41.12  &\underline{78.27}& 72.92& \textbf{95.48} \\

HW&66.72  & 75.26 & 74.25 &79.50 &72.59 &73.87 &83.20 &69.54 &\underline{85.04}& \textbf{95.40}  \\
\hline
%\hline
\end{tabular}
 \caption{ Results of all methods on six datasets. The best result in each row is shown in bold and   the second-best is underlined.}
\label{comparision}
\end{table*}
\begin{table*}[!t]

\centering
\begin{tabular}{|r|c|c|c|c|c|c|c|c|c|c|c|c| }
\hline
  \    & \multicolumn{3}{|c|} {Components}      & \multicolumn{3}{|c|}{MNIST-USPS} & \multicolumn{3}{|c|}{Fashion-MV}  & \multicolumn{3}{|c|}{BDGP}  \\
\hline
\    & $L_r$ &  $L_s$ & $L_Q$    &  ACC   &NMI&ARI              & ACC & NMI&ARI &ACC & NMI&ARI   \\
\hline
(A)    & \ding{51} &  \ding{55}   & \ding{55}    & 64.26 &65.37           &    52.03 &62.80&71.57&55.56     &37.96&36.42&  14.10        \\
\hline
(B)    & \ding{51} &  \ding{51}    &   \ding{55}    & 83.44 & 82.34 &  73.67  &75.12&79.65& 67.08    &61.88&65.13&  46.56         \\
\hline
(C)    &  \ding{51} &   \ding{55}   &   \ding{51}     & 84.10  &90.21 & 82.79 &93.38&90.78&87.55   &71.28&64.34& 56.15  \\ \hline
(D)    &  \ding{51} &   \ding{51}   &   \ding{51}     & 99.58 & 98.86 & 99.07 &97.89&95.40&95.48  &98.00 &94.69  & 95.11   \\

\hline
\end{tabular}
 \caption{ Ablation studies on DMCAG.  }
% “ \ding{51}” represents that the loss component is included in the training, and correspondingly “ \ding{55}” means that the loss component is not included.
 \label{ab}
\end{table*}

\paragraph{Datasets.}
As shown in Table \ref{datasets}, our experiments are carried out on six datasets.
Specifically, \textbf{MNIST-USPS} \cite{peng2019comic} collects from two handwritten digital image datasets, which are treated as two views.
\textbf{Multi-COIL-10} \cite{xu2021multi} collects 720 grayscale object images with size of $32 \times 32$ from 10 clusters, where the different views  represent  various  poses of objects.
\textbf{BDGP} \cite{cai2012joint}
contains  5 different types of drosophila. Each sample has the visual and textual views.
\textbf{UCI-digits}\footnote{https://archive.ics.uci.edu/ml/datasets/Multiple\%2BFeatures} is a collection of 2000 samples with 3 views, which has 10 categories.
\textbf{Fashion-MV}  \cite{xiao2017fashion}  contains images from 10 categories, where we treat three different styles of one object as three views.
\textbf{Handwritten Numerals (HW)}\footnote{https://archive.ics.uci.edu/ml/datasets.php} contains 2000 samples from 10 categories corresponding to numerals 0-9. Each sample has six visual views.

\paragraph{Comparison Methods.} Comparsion methods include 3 traditional single-view clustering methods, i.e., $k$-means \cite{kmeans1967}, SC (Spectral clustering \cite{ng2001spectral}), and DEC (Deep embedded clustering \cite{xie2016unsupervised}),  and 6 state-of-art  MVC methods, i.e., CSMSC (Consistent and specific multi-view subspace clustering\cite{luo2018consistent}), FMR (Flexible multi-view representation learning for subspace clustering), SAMVC (Self-paced and auto-weighted multi-view clustering \cite{ren2020self}),
LMVSC (Large-scale multi-view subspace clustering in linear time \cite{kang2020large}), CGMSC (Multi-view subspace clustering with adaptive locally consistent graph regularization \cite{liu2021multi}), and
 FMVACC (Fast multi-view anchor-correspondence clustering  \cite{wang2022align}). 
 %EOMSC-CA(Efficient one-pass multi-view subspace clustering with consensus anchors\cite{liu2022efficient}),
\paragraph{Evaluation Metrics.}
We evaluate the effectiveness of clustering by three commonly used metrics, i.e., clustering accuracy (ACC), normalized mutual information (NMI), and adjusted rand index (ARI). A higher value of each evaluation metric indicates a better clustering performance.
\paragraph{Implementation.}
The convolutional (Conv) and fully connected (Fc) neural
networks are applied according to different types of data.
For the image datasets, i.e., MNIST-USPS and  Multi-COIL-10,  we use the convolutional autoencoder (CAE) for each view to learn embedded features. The encoder is  Input-Conv$^{32}_4$-Conv$^{64}_4$-Conv$^{64}_4$-Fc$_{10}$.  
Since all views of BDGP, UCI-digits, Fashion-MV, and HW  are vector data, we utilize fully connected autoencoder. 
For each view, the encoder is Input-Fc$_{500}$-Fc$_{500}$-Fc$_{2000}$-Fc$_{10}$. All the decoders are symmetric with the corresponding encoders. Following \cite{kang2020large}, we select anchor numbers in the range [10, 100].  
We select $\gamma$ from \{0.1, 1, 10\}.  Temperature parameter $\tau$ is set to 1 and $\alpha$ is set to 0.001 for all experiments.
%Additionally, some factors, such as raw features' dimensionality, noise level  of different data sets could also influence the anchor number. 
All experiments  are performed on Windows PC with Intel (R) Core (TM) i5-12600K CPU@3.69 GHz, 32.0 GB RAM, and GeForce RTX 3070ti GPU (8 GB caches).  For fair comparison, all baselines are tuned to the best performance according to the corresponding papers. 

\begin{figure*}[!t]
	\centering
	\begin{subfigure}{0.328\linewidth}
		\centering
		\includegraphics[width=0.9\linewidth]{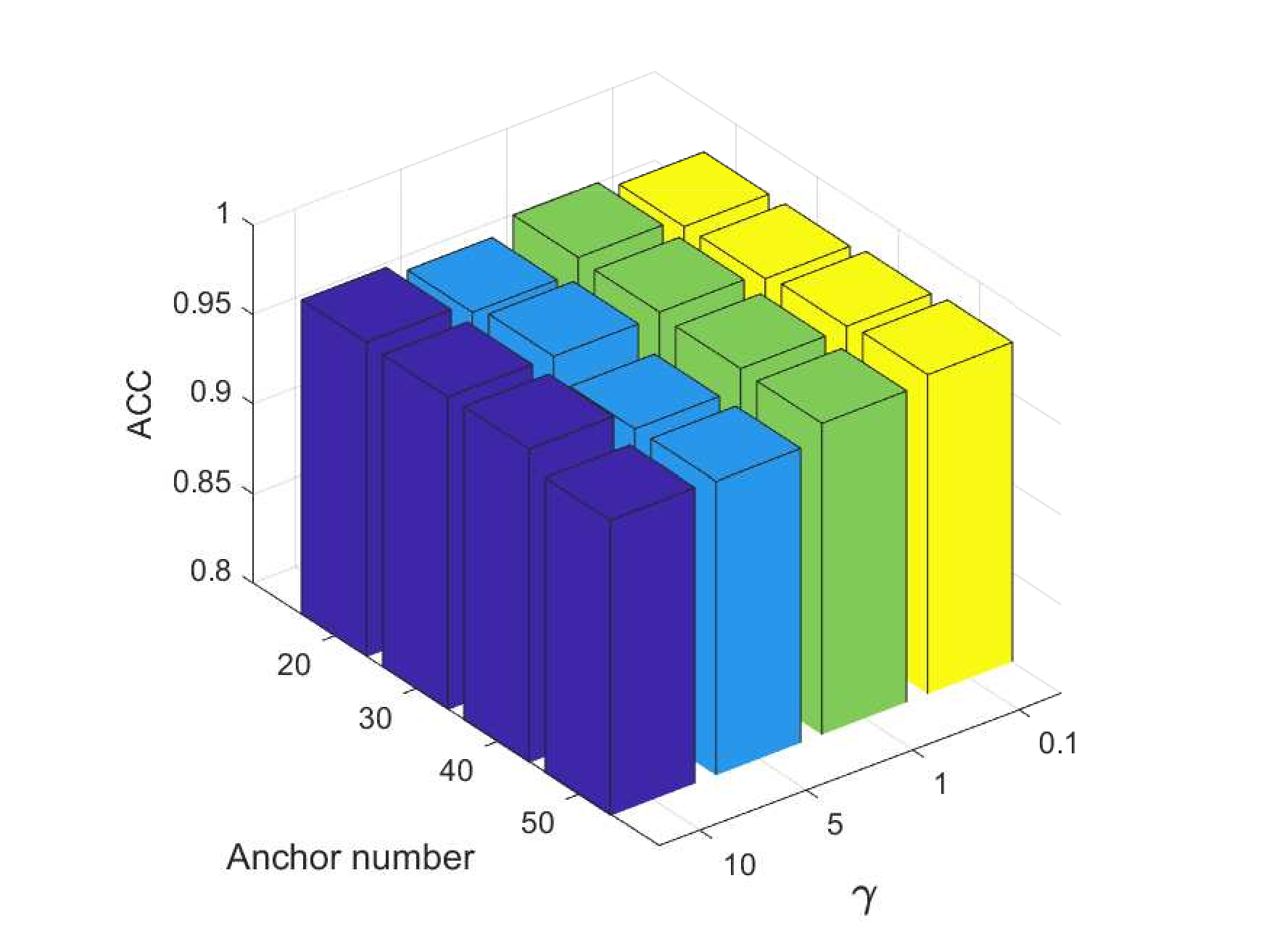}
		\label{}%文中引用该图片代号
	\end{subfigure}
	\centering
	\begin{subfigure}{0.328\linewidth}
		\centering
		\includegraphics[width=0.9\linewidth]{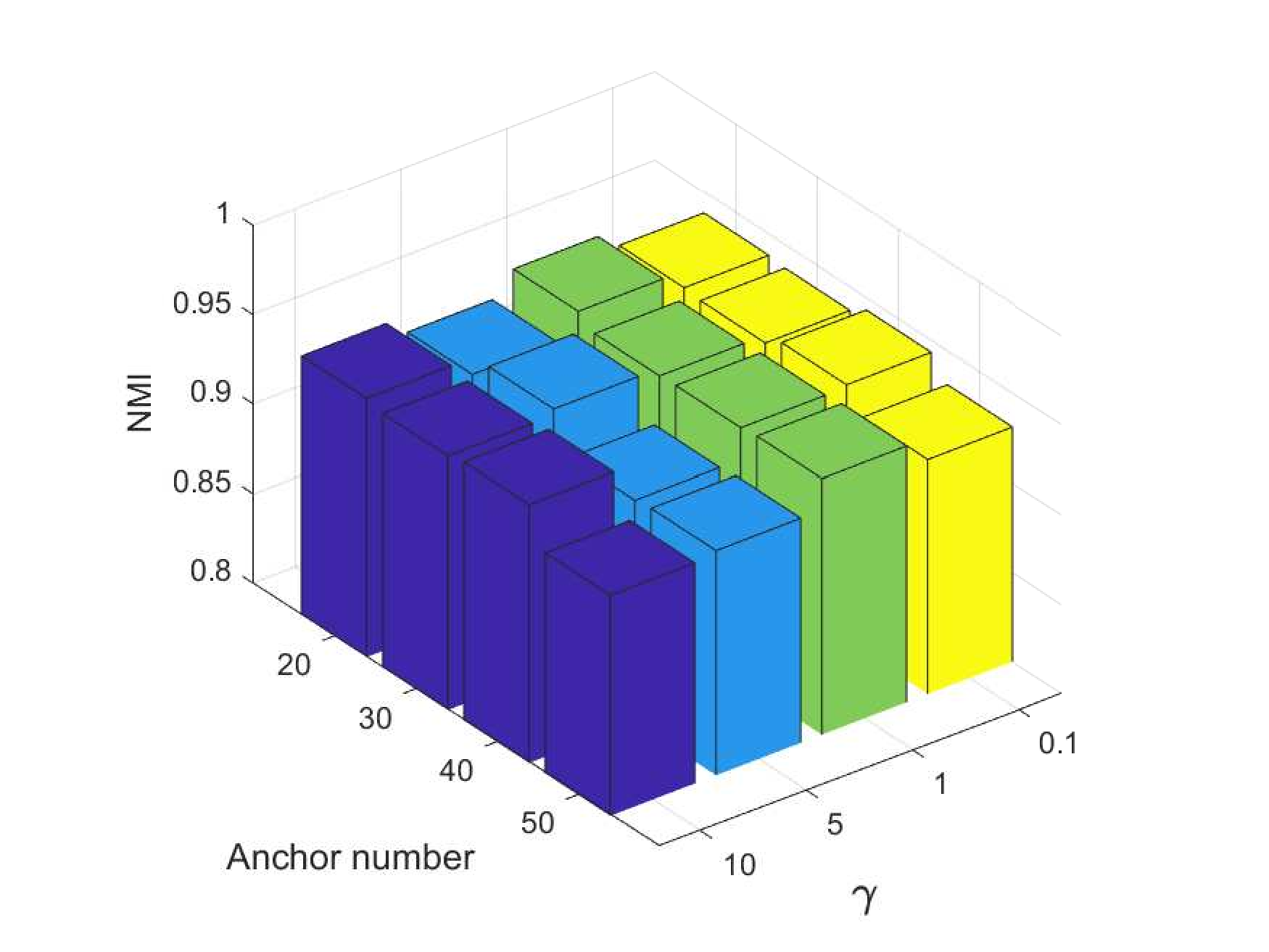}
		\label{}%文中引用该图片代号
	\end{subfigure}
	\centering
	\begin{subfigure}{0.328\linewidth}
		\centering
		\includegraphics[width=0.9\linewidth]{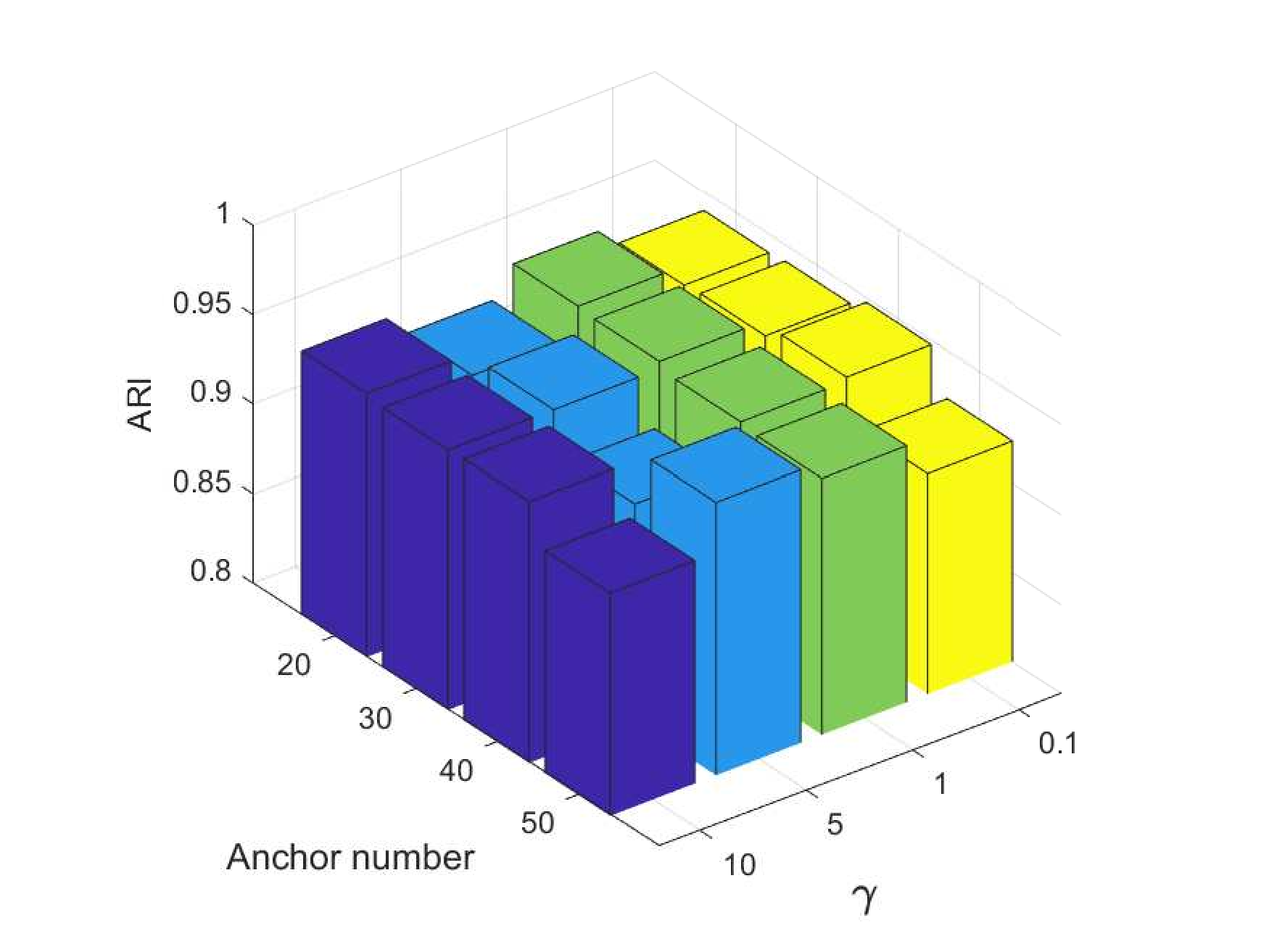}
		\label{}%文中引用该图片代号
	\end{subfigure}
	\caption{Clustering performance with different parameter settings on HW dataset.}
	\label{parameter}
\end{figure*}

\begin{figure*}[!t]
	\centering
	\begin{subfigure}{0.24\linewidth}
		\centering
		\includegraphics[width=1.1\linewidth]{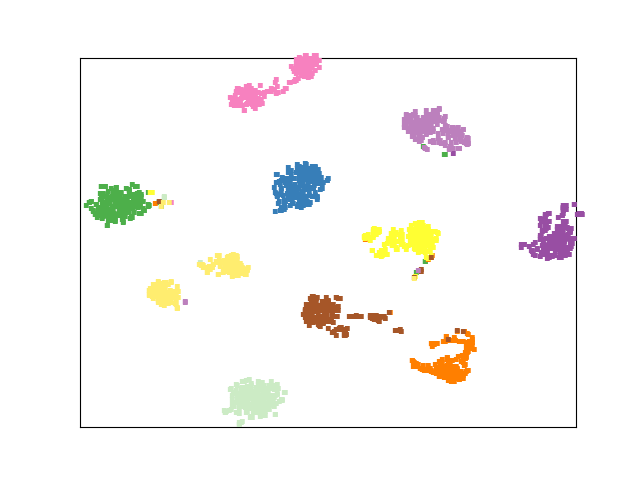}
		\label{v1}%文中引用该图片代号
        \caption{HW}
	\end{subfigure}
	\centering
	\begin{subfigure}{0.24\linewidth}
		\centering
		\includegraphics[width=1.1\linewidth]{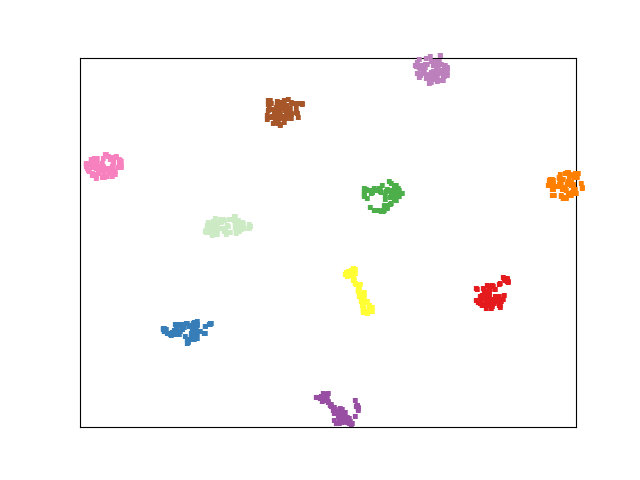}
		\label{v2}%文中引用该图片代号
           \caption{Multi-COIL-10}
	\end{subfigure}
	\centering
	\begin{subfigure}{0.24\linewidth}
		\centering
		\includegraphics[width=1.1\linewidth]{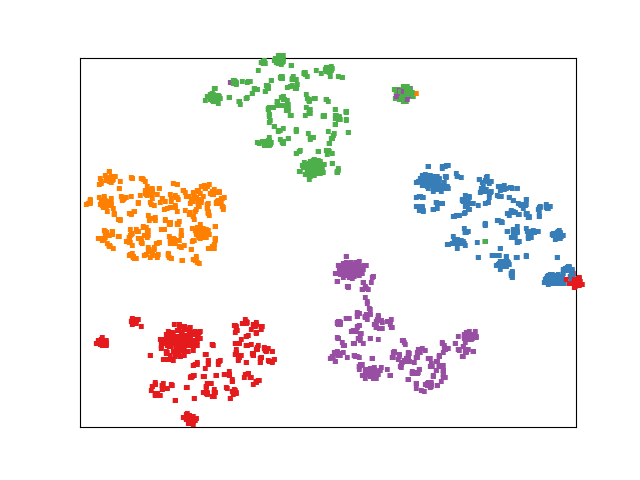}
		\label{v3}%文中引用该图片代号
        \caption{BDGP}
	\end{subfigure}
	\centering
	\begin{subfigure}{0.24\linewidth}
		\centering
		\includegraphics[width=1.1\linewidth]{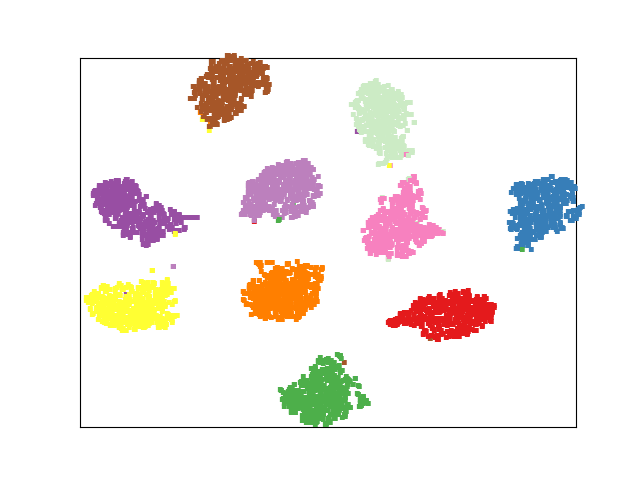}
		\label{v4}%文中引用该图片代号
           \caption{MNIST-USPS}
	\end{subfigure}
\caption{Visualization of the clustering results on four datasets.}
 \label{vis}
 \end{figure*}
\subsection{Results and Analysis}

\paragraph{Comparison with  Baselines.}
The comparison of  DMCAG and baseline methods  is shown in Table \ref{comparision}, where the best result is bolded in each row and the second-best  result is underlined. Due to the high complexity, the result of FMR is not obtained on Fasion-MV dataset after running 24 hours.  From Table \ref{comparision}, we can observe that  the proposed method achieves the best performance among all baseline models on six datasets, illustrating the validity of our method. Especially for Multi-COIL-10 and Fasion-MV, our method makes greater improvement than existing methods. The main reason is that we conduct clearer self-supervised learning with latent anchor graphs and achieve clustering consistency via contrastive learning on pseudo-labels.
Besides, many subspace clustering methods do not perform well on BDGP dataset due to the high dimension of one view's feature, however, our method outperforms due to the ability of extracting features from high-dimension data. 
\paragraph{Ablation Studies.}
To  verify the effectiveness of  the proposed method, we further conduct a set of ablation studies on the loss components from Eqs. (\ref{l_c}) and (\ref{l_q}). $L_r$ is the reconstruction loss of autoencoders. $L_s$ aims to mine global information and  supervise the learning process with global target distribution. $L_Q$  is optimized to achieve the consistency  among all views. In the Table \ref{ab},
(A) is optimized without self-supervised learning and contrastive learning, (B) is optimized without contrastive learning, and (C) is optimized without self-supervised process. 
As shown in Table \ref{ab}, we can find that  (B) and (C) perform better than (A). (D) (our proposed model with all loss components) performs the best, illustrating that  each loss component is important for the final clustering result. %Especially $L_Q$ 
%\paragraph{Parameter Sensitivity Analysis.}

\paragraph{Parameter Sensitivity Analysis.}
As   main   hyperparameters of the  proposed method are  anchor number and sparsity coefficient, we test the general  clustering performance with different settings to show the stability of DMCAG.
From Fig. \ref{parameter} we can observe that our method in a certain range of parameters is insensitive to the clustering  results.  In addition, we can find that too many anchors will reduce clustering performance as the anchors'  information  may not be  representative.  A too large or too small $\gamma$  also reduces the final clustering performence, of which the reason is that the  inappropriate sparsity constraint could lead to extra errors.
% for learning subspace, so the suggested $\gamma$ is 5.
\paragraph{Visualization of Clustering Results.}
We visualize the clustering results on four datasets  via t-SNE \cite{van2008visualizing}, which reduces the dimension of the extracted feature vectors to 2D. As shown in Fig. \ref{vis}, where different colors denote the labels of different nodes,  we can find that the final  clustering structures are clearly visible, especially for Multi-COIL-10 and MNIST-USPS, which further  demonstrates  the effectiveness of our DMCAG method.    

\section{Conclusion}

In this paper, we proposed a  novel  DMCAG framework    for   deep multi-view subspace clustering. 
By introducing the latent anchor graph and  spectral self-supervised learning, we can effectively perform spectral clustering to obtain more robust  global target distribution    and   significantly improve the latent features  structure for clustering.  In addition, we adopt the contrastive learning on   the  soft assignment  of  each view to  achieve the  consistency  among all views. Extensive experiments on six multi-view data sets  and ablation studies have demonstrated  the  superiority of the proposed method.  Moreover,   it is an interesting future work to  further improve the efficiency of constructing the 
latent anchor subspace  to deal with  large-scale data, and find more robust target distribution.

%%%%%%%%% REFERENCES
\section*{Acknowledgment}
This work was supported in part by Sichuan Science and Technology Program (Nos. 2022YFS0047 and 2022YFS0055), Medico-Engineering Cooperation Funds from the University of Electronic Science and Technology of China (No. ZYGX2021YGLH022), National Science Foundation (MRI 2215789), and Lehigh’s grants (Nos. S00010293 and 001250).
%This work was supported in part by National Key Research
%and Development Program of China (Nos. 2020YFC2004300
%and 2020YFC2004302), National Natural Science Founda-
%tion of China (No. 61971052), Sichuan Science and Tech-
%nology Program (Nos. 2022YFS0047 and 2022YFS0055),
%Medico-Engineering Cooperation Funds from University
%of Electronic Science and Technology of China (No.
%ZYGX2021YGLH022), National Science Foundation (Nos.
%III-1763325, III-1909323, III-2106758, SaTC-1930941, and
%MRI 2215789), and Lehigh’s grants (Nos. S00010293 and
%001250).
\bibliographystyle{named}
\bibliography{ijcai23}

\end{document}